\def\year{2019}\relax
\documentclass[letterpaper]{article} 
\usepackage{nopageno}
\usepackage[letterpaper, total={7.1in, 9in}]{geometry}
\usepackage{times}  
\usepackage{helvet}  
\usepackage{courier}  
\usepackage{url}  
\usepackage{graphicx}  
\usepackage{booktabs} 
\usepackage{amsmath,amssymb,amsthm}
\usepackage{color, xcolor, colortbl}
\usepackage{natbib}
\usepackage{bbm}
\usepackage{booktabs}
\usepackage{bm}
\usepackage{subcaption}
\usepackage{forest}
\usepackage{qtree}
\usepackage{thm-restate}
\usepackage{mathtools}
\usepackage{cprotect}

\usepackage{threeparttable}

\usepackage[breakall]{truncate}

\usepackage{xcolor}

\theoremstyle{plain}

\newtheorem{lemma}{Lemma}

\theoremstyle{definition}

\theoremstyle{remark}

\newcommand\defeq{\stackrel{\mathclap{\small\mbox{def}}}{=}}
\newcommand{\reals}{\mathbb{R}}

\newcommand{\calT}{\mathcal{T}}

\renewcommand \vec [1]{\bm{#1}}
\renewcommand \bar [1]{\overline{\bm{#1}}}

\newcommand \vmu {\vec{\mu}}

\newcommand \vv {\vec{v}}
\newcommand \vw {\vec{w}}

\newcommand \mean {\mathrm{mean}}

\newcommand\blfootnote[1]{%
  \begingroup
  \renewcommand\thefootnote{}\footnote{#1}%
  \addtocounter{footnote}{-1}%
  \endgroup
}

\begin{document}
\title{What are the biases in my word embedding?}
\author{ 
Nathaniel Swinger$^*$\\
Lexington High School
\and
Maria De-Arteaga$^*$\\
Carnegie Mellon University
\and
Neil Thomas Heffernan IV\\
Shrewsbury High School
\and
Mark DM Leiserson\\
University of Maryland
\and
Adam Tauman Kalai\\
Microsoft Research
}
\date{Jan 28, 2019}
\maketitle

\setlength\fboxrule{1pt}
\noindent\fbox{\begin{minipage}{\dimexpr\textwidth-3\fboxsep-2\fboxrule\relax}
{\bf Warning: This paper includes stereotypes and terms which are offensive in nature. It is important, however, for researchers to be aware of the biases contained in word embeddings. 
}
\end{minipage}}
\begin{abstract}


This paper presents an algorithm for enumerating biases in word embeddings. The algorithm exposes a large number of offensive associations related to sensitive features such as race and gender on publicly available embeddings, including a supposedly ``debiased'' embedding. These biases are concerning in light of the widespread use of word embeddings. The associations are identified by geometric patterns in word embeddings that run parallel between people's names and common lower-case tokens. The algorithm is highly unsupervised: it does not even require the sensitive features to be pre-specified. This is desirable because: (a) many forms of discrimination--such as racial discrimination--are linked to social constructs that may vary depending on the context, rather than to categories with fixed definitions; and (b) it makes it easier to identify biases against {\em intersectional} groups, which depend on combinations of sensitive features. The inputs to our algorithm are a list of target tokens, e.g. names, and a word embedding. It outputs a number of Word Embedding Association Tests (WEATs) that capture various biases present in the data. We illustrate the utility of our approach on publicly available word embeddings and lists of names, and evaluate its output using crowdsourcing. We also show how removing names may not remove potential proxy bias.

\end{abstract}

\section{Introduction}
\blfootnote{$^*$ Indicates equal contribution.}

Bias in data representation is an important element of fairness in Artificially Intelligent systems \citep{barocas2017representational, Caliskan17:Semantics, zemel2013fairrepresentations, Dwork:2012}. We consider the problem of {\em Unsupervised Bias Enumeration} (UBE): discovering biases automatically from an unlabeled data representation. There are multiple reasons why such an algorithm is useful. First, social scientists can use it as a tool to study human bias, as data analysis is increasingly common in social studies of human biases \citep{garg2018word, kozlowski2018geometry}. Second, finding bias is a natural step in ``debiasing'' representations \citep{Bolukbasi16:Man}. Finally, it can help in avoiding systems that perpetuate these biases: problematic biases can raise red flags for engineers, who can choose to not use a representation or watch out for certain biases in downstream applications, while little or no bias can be a useful green light indicating that a representation is usable. While deciding which biases are problematic is ultimately application specific, UBE may be useful in a ``fair ML" pipeline.

We design a UBE algorithm for word embeddings, which are commonly used representations of tokens (e.g. words and phrases) that have been found to contain harmful bias~\citep{Bolukbasi16:Man}. Researchers linking these biases to human biases proposed the Word Embedding Association Test (WEAT)~\citep{Caliskan17:Semantics}. The WEAT draws its inspiration from the Implicit Association Test (IAT), a widely-used approach to measure human bias~\citep{greenwald1998measuring}. An IAT $\calT=(X_1,A_1,X_2,A_2)$ compares two sets of {\em target tokens} $X_1$ and $X_2$, such as female vs.\ male names, and a pair of opposing sets of {\em attribute tokens} $A_1$ and $A_2$, such as workplace vs.\ family-themed words. Average differences in a person's response times when asked to link tokens that have anti-stereotypical vs.\ stereotypical relationships have been shown to indicate the strength of association between concepts. Analogously, the WEAT 
uses vector similarity across pairs of tokens in the sets to measure association strength. As in the case of the IAT, the inputs for a WEAT are sets of tokens $\calT$ predefined by researchers.

Our UBE algorithm takes as input a word embedding and a list of target tokens, and {\em outputs} numerous tests $\calT_1, \calT_2,\ldots,$ that are found to be statistically significant by a method we introduce for bounding false discovery rates. A crowdsourcing study of tests generated on three publicly-available word embeddings and a list of names from the Social Security Administration confirms that the biases enumerated are largely consistent with human stereotypes. The generated tests capture racial, gender, religious, and age biases, among others. Table \ref{table:offensive} shows the name/word associations output by our algorithm that were rated most offensive by crowd workers. 

Creating such tests automatically has several advantages. First, it is not feasible to manually author all possible tests of interest. Domain experts normally create such tests, and it is unreasonable to expect them to cover all possible groups, especially if they do not know which groups are represented in their data. For example, a domain expert based on the United States may not think of testing for caste discrimination, hence biases that an embedding may have against certain Indian last names may go unnoticed. Finally, if a word embedding reveals no biases, this is evidence for lack of bias. We test this by running our UBE algorithm on the supposedly debiased embedding of \cite{Bolukbasi16:Man}.

Our approach for UBE leverages two geometric properties of word embeddings, which we call the \textit{parallel} and \textit{cluster} properties. The well-known parallel property indicates that differences between two similar token pairs, such as Mary$-$John and Queen$-$King, are often nearly parallel vectors. This suggests that among tokens in a similar topic or category, those parallel to name differences may represent biases, as was found by \cite{Bolukbasi16:Man} and \cite{Caliskan17:Semantics}. 
The cluster property, which we were previously unaware of, indicates that the (normalized) vectors of names and words cluster into semantically meaningful groups. For names, the clusters capture social structures such as gender, religion, and others. For words, clusters of words include word categories on topics such as food, education, occupations, and sports. We use these properties to design a UBE algorithm that outputs WEATs.

\begin{table*}
    \centering
\setlength\tabcolsep{4.22pt}

{
\begin{threeparttable}
\begin{tabular}{|lll|lll|lll|}
\multicolumn{3}{c}{Word2Vec trained on Google news} & \multicolumn{3}{c}{fastText trained on the Web} & \multicolumn{3}{c}{GloVe trained on the Web}\\
\hline
\textbf{w2v F8} & \textbf{w2v F11} & \textbf{w2v F6} & \textbf{fast F10} & \textbf{fast F7} & \textbf{fast F5} & \textbf{glove F8} & \textbf{glove F7} & \textbf{glove F5}\\
\hline
illegal immigrant & aggravated robbery & subcontinent & n***** & jihad & s****** & turban & cartel & pornstar\\
drug trafficking & aggravated assault & tribesmen & f***** & militants & maid & saree & undocumented & hottie\\
deported & felonious assault & miscreants & dreads & caliphate & busty & hijab & culpable & nubile\\
\hline
\end{tabular}
\end{threeparttable}}
\caption{Terms associated with name groups (see Tables \ref{table:fullw2v_names} and \ref{table:names_other_embeddings} for name groups \textbf{w2v F8}, etc.)  generated from three popular pre-trained word embeddings that were rated by crowd workers as both most offensive and aligned with societal biases. These associations do {\em not} reflect the personal beliefs of the crowd workers or authors of this paper. See Appendix \ref{ap:bleep} for a discussion of the bleep-censored words.
}
    \label{table:offensive}

\end{table*}

Technical challenges arise around any procedure for enumerating biases. First, the combinatorial explosion of comparisons among multiple groups parallels issues in human IAT studies as aptly described by \cite{bluemke2008reliability}:
``The evaluation of
multiple target concepts such as social groups within a multi-ethnic nation \cite[e.g. White vs. Asian Americans, White vs.
African Americans, African vs. Asian Americans;][]{devos2005american}
requires numerous pairwise comparisons for a
complete picture''. We alleviate this problem, paralleling that work on human IATs, by generalizing the WEAT to $n$ groups for arbitrary $n$. The second problem, for any UBE algorithm, is determining statistical significance to account for multiple hypothesis testing. To do this, we introduce a novel rotational null hypothesis specific to word embeddings. Third, we provide a human evaluation of the biases, contending with the difficulty that many people are unfamiliar with some groups of names. 



Beyond word embeddings and IATs, related work in other subjects is worth mention. First, a body of work studies fairness properties of classification and regression algorithms \citep[e.g.][]{Dwork:2012, kearns2017preventing}. While our work does not concern supervised learning, it is within this work that we find one of our main motivations--the importance of accounting for intersectionality when studying algorithmic biases. In particular,  \citet{buolamwini2018gender} demonstrate accuracy disparities in image classification highlighting the fact that the magnitude of biases against an intersectional group may go unnoticed when only evaluating for each protected feature independently. Finally, while a significant portion of the empirical research on algorithmic fairness has focused on the societal biases that are most pressing in the countries where the majority of researchers currently conducting the work are based, the literature also contains examples of biases that may be of particular importance in other parts of the world~\citep{shankar2017no,hoque2017evaluating}. UBE can aspire to be useful in multiple contexts, and enable the discovery of biases in a way that relies less on enumeration by domain experts. 

\section{Definitions}\label{sec:notation}
A $d$-dimensional word embedding consists of a set of tokens $\mathcal{W}$ with a nonzero vector $\vw \in \mathbb{R}^d$  associated with each token $w \in \mathcal{W}$. Vectors are displayed in boldface. As is standard, we refer to the {\em  similarity} between tokens $v$ and $w$ by the cosine of their vector angle, $\cos(\vv,\vw)$.  We write  $\bar{v}=\vv/|\vv|$ to be the vector normalized to unit-length associated with any vector $\vv\in \mathbb{R}^d$ (or 0 if $\vv=0$). This enables us to conveniently write the similarity between tokens $v$ and $w$ as an inner product, $\cos(\vv,\vw)=\bar{\vv}\cdot \bar{\vw}$. For token set $S$, we write $\bar{S}=\sum_{v \in S} \bar{\vv}/|S|$ so that $\bar{S}\cdot \bar{T}=\mean_{v \in S, w \in T} \bar{\vv} \cdot \bar{\vw}$ is the mean similarity between pairs of tokens in sets $S,T$. We denote the set difference between $S$ and $T$ by $S \setminus T$, and we denote the first $n$ whole numbers by $[n]=\{1,2,\ldots,n\}$. 

\subsection{Generalizing Word Embedding Association Tests}\label{sec:SWEAT}

We assume that there is a given set of possible targets $\mathcal{X}$ and attributes $\mathcal{A}$.
Henceforth, since in our evaluation all targets are names and all attributes are lower-case words (or phrases), we refer to targets as names and attributes as words. Nonetheless, in principle, the algorithm can be run on any sets of target and attribute tokens. \cite{Caliskan17:Semantics} define a WEAT statistic for two equal-sized groups of names $X_1, X_2\subseteq \mathcal{X}$ and words $A_1, A_2\subseteq{\mathcal{A}}$ which can be  conveniently written in our notation as,
$$s(X_1, A_1, X_2, A_2) ~\defeq~ \left(\sum_{x \in X_1}\bar{x}-\sum_{x \in X_2}\bar{x}\right) \cdot (\bar{A}_1 -\bar{A}_2).$$

In studies of human biases, the combinatorial explosion in groups can be avoided by teasing apart {\em Single-Category} IATs which assess associations one group at a time \citep[e.g.][]{karpinski2006single, penke2006single, bluemke2008reliability}. In word embeddings, we define a simple generalization for $n \geq 1$, nonempty groups $X_1, \ldots, X_n$  of arbitrary sizes and words $A_1, \ldots, A_n$, as follows:
\begin{align*}
g(X_1, A_1, \ldots, X_n, A_n) ~&\defeq ~~ \sum_{i=1}^n (\bar{X}_i - \vmu) \cdot (\bar{A}_i-\bar{\mathcal{A}})\\   ~~~\text{where }\vmu ~&\defeq ~~\begin{cases}\bar{\mathcal{X}}&\text{for }n=1,\\
\sum_{i} \bar{X}_i/n&\text{for }n\geq 2.\end{cases}\end{align*}
Note that $g$ is symmetric with respect to ordering and weights groups equally regardless of size. The definition differs for $n=1$, otherwise $g\equiv 0$. 

The following three properties motivate this as a ``natural'' generalization of WEAT to one or more groups.
\begin{lemma}\label{lem:decomp}
For any $X_1, X_2\subseteq \mathcal{X}$ of equal sizes $|X_1|=|X_2|$ and any nonempty $A_1, A_2 \subseteq\mathcal{A}$,
\begin{equation*}
s(X_1,A_1,X_2,A_2)=2|X_1|~g(X_1,A_1,X_2,A_2)
\end{equation*}
\end{lemma}
\begin{lemma}\label{lem:decomp2} For any nonempty sets $X \subset \mathcal{X}$, $A\subset\mathcal{A}$, let their complements sets $X^c=\mathcal{X}\setminus X$ and $A^c=\mathcal{A}\setminus A$. Then,
\begin{equation*}
g(X, A) = 2 g(X,A,\mathcal{X},\mathcal{A})= 2\frac{|X^c|}{|\mathcal{X}|} \frac{|A^c|}{|\mathcal{A}|} g(X,A,X^c, A^c)\end{equation*}
\end{lemma}
\begin{lemma}\label{lem:decomp3}
For any $n> 1$ and nonempty $X_1, X_2, \ldots, X_n\subseteq \mathcal{X}$ and $A_1, A_2, \ldots, A_n\subseteq \bar{\mathcal{A}}$,
\begin{equation*}
g(X_1,A_1,\ldots,X_n,A_n) = \sum_{i\in [n]} g(X_i, A_i) - \sum_{i,j\in [n]} \frac{g(X_i, A_j)}{n}\end{equation*}
\end{lemma}
Lemma~\ref{lem:decomp} explains why we call it a generalization: for $n=2$ and equal-sized name sets, the values are proportional with a factor that only depends on the set size. More generally, $g$ can accommodate unequal set sizes and $n\neq 2$. 

Lemma~\ref{lem:decomp2} shows that for $n=1$ group, the definition is proportional the WEAT with the two groups $X$ vs.~all names $\mathcal{X}$ and words $A$ vs.~$\mathcal{A}$. Equivalently, it is proportional to the WEAT between $X$ and $A$ and their compliments.

Finally, Lemma~\ref{lem:decomp3} gives a {\em decomposition} of a WEAT into $n^2$ single-group WEATs $g(X_i,A_j)$. In particular, the value of a single multi-group WEAT reflects a combination of the $n$ association strengths between $X_i$ and $A_i$, and $n^2$ disassociation strengths between $X_i$ and $A_j$.  
As discussed on the literature on IATs, a large effect could reflect a strong association between $X_1$ and $A_1$ or $X_2$ and $A_2$, a strong disassociation between $X_1$ and $A_2$ or $X_2$ and $A_1$, or some combination of these factors. Proofs are deferred to Appendix \ref{ap:proofs}.

\section{Unsupervised Bias Enumeration algorithm}\label{sec:alg}

The inputs to our UBE algorithm are shown in Table~\ref{tab:constants}. The output is $m$ WEATs, each with $n$ groups with associated sets of words and statistical confidences (p-values) in $[0,1]$. Each WEAT has words from a single category, but several of the $m$ WEATs may yield no significant associations. 



At a high level, the algorithm follows a simple structure. It selects $n$ disjoint groups of names $X_1,\ldots, X_n \subset \mathcal{X}$, and $m$ disjoint categories of lower-case words $\mathcal{A}_1,\ldots, \mathcal{A}_m$. All WEATs share the same $n$ name groups, and each WEAT has words from a single category $\mathcal{A}_j$, with $t$ words associated to each $X_i$. Thus the WEATs can be conveniently visualized in a tabular structure.

For convenience, we normalize all word embedding vectors to be unit length. Note that we only compute cosines between them, and the cosine is simply the inner product for unit vectors. We now detail the algorithm's steps. 

\subsection{Step 1: Cleaning names and defining groups}\label{sec:clean}
We begin with a set of names\footnote{While the set of names is an input to our system, they could also be extracted from the embedding itself.} $\mathcal{X}$, e.g., frequent first names from a database. Since word embeddings do not differentiate between words that have the same spelling but different meanings, we first ``clean'' the given names to remove names such as ``May'' and ``Virginia'', whose embeddings are more reflective of other uses, such as a month or verb and a US state. Our cleaning procedure, detailed in Appendix C, is similar to that of \cite{Caliskan17:Semantics}.

We then use K-means++ clustering \citep[from scikit-learn,][with default parameters]{scikit-learn} 
 to cluster the normalized word vectors of the names, yielding groups  $X_1\cup \ldots\cup X_n=\mathcal{X}$. Finally, we define $\mu = \sum_i \bar{X}_i/n$.

\begin{table}
    \centering
    \begin{tabular}{r|l|l}
        name & meaning & default\\\hline
        {\em WE} & word embedding & \verb|w2v| \\
        $\mathcal{X}$ & set of names & SSA\\
        $n$ & number of target groups & 12\\
        $m$ & number of categories & 64\\
        $M$ & number of frequent lower-case words & 30,000\\
        $t$ & number of words per WEAT & 3\\
        $\alpha$ & false discovery rate & 0.05\\
    \end{tabular}
    \smallskip
    \caption{Inputs to the UBE algorithm.}
    \label{tab:constants}
\end{table}


\subsection{Step 2: Defining word categories} 
To define categories, we cluster the most frequent $M$ lower-case tokens in the word embedding into $m$ clusters using K-means++, yielding clusters of categories $\mathcal{A}_1,\ldots, \mathcal{A}_m$. The constant $M$ is chosen to cover as many recognizable words as possible without introducing too many unrecognizable tokens. 
As we shall see, categories capture concepts such as occupations, food-related words, and so forth.  

\subsection{Step 3: Selecting words $A_{ij}\subset \mathcal{A}_j$} 
A test $\mathcal{T}_j=(X_1,A_{1j}, \ldots,X_n,A_{nj})$ is chosen with disjoint $A_{ij}\subset \mathcal{A}_{j}$, each of size $t=|A_{ij}|$. To ensure disjointness,\footnote{
If multiplicities are desired, the Voronoi sets $V_{ij}$ could be omitted, optimizing $A_{ij}\subset \mathcal{A}_j$ directly.} $\mathcal{A}_j$ is first partitioned into $n$ ``Voronoi'' sets $V_{ij}\subseteq \mathcal{A}_j$ consisting of the words whose embedding is closest to each corresponding center $\bar{X}_i$, i.e.,
$$V_{ij} = \left\{w \in \mathcal{A}_j~|~ i = \arg\max_{i'\in [n]} \bar{w} \cdot \bar{X}_{i'}\right\}$$
It then outputs $A_{ij}$ defined as the $t$ words maximizing the following:
$$\max_{w \in V_{ij}} (\bar{X}_i-\mu)\cdot (\bar{w}-\bar{\mathcal{A}}_j)$$ 

The more computationally-demanding step is to compute, using Monte Carlo sampling, the $n$ p-values for $\mathcal{T}_j$, as described next. 

\subsection{Step 4: Computing p-values and ordering}
To test whether the associations we find are larger than one would find if there was no relationship between the names $X_i$ and words $\mathcal{A}$, we consider the following ``\textbf{rotational null hypothesis}'': the words in the embedding are generated through some process in which the alignment between names and words is random. This is formalized by imagining that a random rotation was applied (multiplying by a uniformly Haar random orthogonal matrix $U$) to the word embeddings but not to the name embeddings.

Specifically, to compute p-value $p_{ij}$ for each $(X_i, A_{ij})$, we first compute a score $\sigma_{ij} = (\bar{X}_i-\mu)\cdot(\bar{A}_{ij}-\bar{\mathcal{A}})$.  We then compute $R=10,000$ uniformly random orthogonal rotations $U_1, \ldots, U_R \in \reals^{d \times d}$, drawn according to the Haar measure. For each rotation, we simulate running our algorithm as if the name embeddings were transformed by $U$ (while the word embeddings remain as is). For each rotation $U_r$, the sets $A_{ijr}$ chosen to maximize $(\bar{X}_iU_r-\mu U_r)\cdot (\bar{w}-\bar{\mathcal{A}}_j)$, and the corresponding $V_{ijr}$ and the resulting $\sigma_{ijr}$ are computed. Finally, $p_{ij}$ is the fraction of rotations for which the score $\sigma_{ijr}\geq \sigma_{ij}$ (plus an add-1 penalty standard for Monte Carlo p-values). 


Furthermore, since the algorithm outputs many (hundreds) of name/word biases, the Benjamini-Hochberg  \citeyearpar{BenjaminiHochberg95} procedure is used to determine a critical p-value that guarantees an $\alpha$ bound on the rate of false discoveries.
Finally, to choose an output ordering on significant tests, 
the $m$ tests are then sorted by the total scores $\sigma_{ij}$ over the pairs determined significant.

\section{Evaluation}

To illustrate the performance of the proposed system in discovering associations, we use a database of first names provided by the Social Security Administration (SSA), which contains number of births per year by sex (F/M)~\citep{SSA}. Preprocessing details are in Appendix C.

We use three publicly available word embeddings, each with $d=300$ dimensions and millions of words:  \verb|w2v|, released in 2013 and trained on approximately 100 billion words from Google News \citep{mikolov2013distributed}, \verb|fast|, trained on 600 billion words from the Web \citep{mikolov2018advances}, and \verb|glove|, also trained on the Web using the GloVe algorithm \citep{pennington2014glove}.



While it is possible to display the three words in each $A_{ij}$, the hundreds or thousands of names in each $X_i$ cannot be displayed in the output of the algorithm. Instead, we use a simple greedy heuristic to give five ``illustrative'' names for each group, which are displayed in the tables in this paper and in our crowdsourcing experiments.  The $k+1^\text{st}$ name shown is chosen, given the first $k$ names, so as to maximize the average similarity of the first $k+1$ names to that of the entire set $X_i$. Hence, the first name is the one whose normalized vector is most central (closest to the cluster mean), the second name is the one which when averaged with the first is as central as possible, and so forth. 

The WEATs can be evaluated in terms of the quality of the name groups and also their associations with words. 
A priori, it was not clear whether clustering name embeddings would yield any name groups or word categories of interest. For all three embeddings we find that the clustering captures latent groups defined in terms of race, age, and gender (we only have binary gender statistics), as illustrated in Table \ref{table:fullw2v_names} for $n=12$ clusters.  
While even a few clusters suffice to capture some demographic differences, more clusters yield much more fine-grained distinctions. For example, with $n=12$ one cluster is of evidently Israeli names (see column I of table \ref{table:fullw2v_names}), which one might not consider predefining a priori since they are a small minority in the U.S. Table~\ref{table:names_other_embeddings} in the Appendix shows demographic composition of clustering for other embeddings. Note that, although we do not have religious statistics for the names, several of the words in the generated associations are religious in nature, suggesting religious biases as well.

\begin{table*}
{\small
\setlength\tabcolsep{1pt}
\begin{tabular*}{\textwidth}{@{\extracolsep{\fill}}rrrrrrrrrrrr}
\toprule
\textbf{w2v~F1} & \textbf{w2v~F2} & \textbf{w2v~F3} & \textbf{w2v~F4} & \textbf{w2v~F5} & \textbf{w2v~F6} & \textbf{w2v~F7} & \textbf{w2v~F8} & \textbf{w2v~F9} & \textbf{w2v~F10} & \textbf{w2v~F11} & \textbf{w2v~F12} \\
\midrule
Amanda & Janice & Marquisha & Mia & Kayla & Kamal & Daniela & Miguel & Yael & Randall & Dashaun & Keith \\
Renee & Jeanette & Latisha & Keva & Carsyn & Nailah & Lucien & Deisy & Moses & Dashiell & Jamell & Gabe \\
Lynnea & Lenna & Tyrique & Hillary & Aislynn & Kya & Marko & Violeta & Michal & Randell & Marlon & Alfred \\
Zoe & Mattie & Marygrace & Penelope & Cj & Maryam & Emelie & Emilio & Shai & Jordan & Davonta & Shane \\
Erika & Marylynn & Takiyah & Savanna & Kaylei & Rohan & Antonia & Yareli & Yehudis & Chace & Demetrius & Stan \\
~~~+581 & ~~~+840 & ~~~+692 & ~~~+558 & ~~~+890 & ~~~+312 & ~~~+391 & ~~~+577 & ~~~+120 & ~~~+432 & ~~~+393 & ~~~+494 \\
\midrule
98\% F & 98\% F & 89\% F & 85\% F & 78\% F & 65\% F & 59\% F & 56\% F & 40\% F & 27\% F &  5\% F &  4\% F \\
\hline
1983 & 1968 & 1978 & 1982 & 1993 & 1991 & 1985 & 1986 & 1989 & 1981 & 1984 & 1976 \\
\hline
 4\% B &  8\% B & 48\% B & 10\% B &  2\% B &  7\% B &  4\% B &  2\% B &  5\% B & 10\% B & 32\% B &  6\% B \\
 4\% H &  4\% H &  3\% H &  9\% H &  1\% H &  4\% H &  9\% H & 70\% H & 10\% H &  3\% H &  5\% H &  3\% H \\
 3\% A &  3\% A &  1\% A & 11\% A &  1\% A & 32\% A &  4\% A &  8\% A &  5\% A &  4\% A &  3\% A &  5\% A \\
89\% W & 84\% W & 47\% W & 69\% W & 95\% W & 56\% W & 83\% W & 21\% W & 79\% W & 83\% W & 59\% W & 86\% W \\
\midrule
\end{tabular*}

}
\cprotect\caption{Illustrative first names (greedily chosen) for $n=12$ groups on the \verb|w2v| embedding. Demographic statistics (computed a
posteriori) are also shown though were not used in generation, including percentage female (at birth), mean year of birth, and
percentage Black, Hispanic, Asian/Pacific Islander, and White. 
}\label{table:fullw2v_names}
\end{table*}


Table \ref{table:debiased} in the appendix shows the biases found in the ``debiased'' \verb|w2v| embedding of \cite{Bolukbasi16:Man}. While the name clusters still exhibit strong binary gender differences, many fewer statistically significant associations were generated for the most gender-polarized clusters.  

\begin{table}
    \centering
    \begin{tabular}{l|lll}
        Emb. & \# significant & \% accurate & \% offensive \\
        \hline
         \verb|w2v| & 235 & 72\% & 35\%\\
         \verb|fast| & 160 & 80\% & 38\%\\
         \verb|glove| & 442 & 48\% & 24\%\\
    \end{tabular}
    \cprotect\caption{Summary statistics for the WEATs generated using the
    three embeddings ($n=12$, $m=64$). The total number of significant name/word associations, the fraction with which the crowd's choice of name group agreed with that of the 
    generated WEAT (accuracy) among the top-12 WEATs, and the fraction rated as offensive.}
    \label{table:summary}
\end{table}

\begin{table*}
{\fontsize{8.1}{9}\selectfont
\setlength\tabcolsep{1pt}
\begin{tabular*}{\textwidth}{@{\extracolsep{\fill}}lllllllllll}
\toprule
\textbf{w2v~F1} & \textbf{w2v~F2} & \textbf{w2v~F3} & \textbf{w2v~F4} & \textbf{w2v~F5} & \textbf{w2v~F6} & \textbf{w2v~F7} & \textbf{w2v~F8} & \textbf{w2v~F9} & \textbf{w2v~F11} & \textbf{w2v~F12} \\
\midrule
 & \cellcolor{orange!10}cookbook, & \cellcolor{orange!10}sweet &  &  & \cellcolor{orange!10}saffron, & \cellcolor{orange!10}mozzarella, & \cellcolor{orange!10}tortillas, & \cellcolor{orange!10}kosher, & \cellcolor{orange!10}fried      & \cellcolor{orange!10}beef,\\
 & \cellcolor{orange!10}baking, & \cellcolor{orange!10}potatoes, &  &  & \cellcolor{orange!10}halal, & \cellcolor{orange!10}foie gras, & \cellcolor{orange!10}salsa, & \cellcolor{orange!10}hummus, & \cellcolor{orange!10}chicken, & \cellcolor{orange!10}beer,\\
 & \cellcolor{orange!10}baked goods & \cellcolor{orange!10}macaroni, &  &  & \cellcolor{orange!10}sweets & \cellcolor{orange!10}caviar & \cellcolor{orange!10}tequila & \cellcolor{orange!10}bagel & \cellcolor{orange!10}crawfish, & \cellcolor{orange!10}hams\\
 &  & \cellcolor{orange!10}green beans &  &  &  &  &  &  & \cellcolor{orange!10}grams & \\
\hline
\cellcolor{orange!10}herself, & husband, & aunt, & hubby, & \cellcolor{orange!10}twin sister, & \cellcolor{orange!10}elder brother, &  &  & \cellcolor{orange!10}bereaved, & \cellcolor{orange!10}younger & \cellcolor{orange!10}buddy,\\
\cellcolor{orange!10}hers, & homebound, & niece, & socialite, & \cellcolor{orange!10}girls, & \cellcolor{orange!10}dowry, &  &  & \cellcolor{orange!10}immigrated, & \cellcolor{orange!10}brother, & \cellcolor{orange!10}boyhood,\\
\cellcolor{orange!10}moms & grandkids & grandmother & cuddle & \cellcolor{orange!10}classmate & \cellcolor{orange!10}refugee camp &  &  & \cellcolor{orange!10}emigrated & \cellcolor{orange!10}twin brother, & \cellcolor{orange!10}fatherhood\\
 &  &  &  &  &  &  &  &  & \cellcolor{orange!10}mentally & \\
 &  &  &  &  &  &  &  &  & \cellcolor{orange!10}r******** & \\
\hline
hostess, & \cellcolor{orange!10}registered &  & \cellcolor{orange!10}supermodel, & helper, & \cellcolor{orange!10}shopkeeper, &  & \cellcolor{orange!10}translator, &  & cab driver, & \cellcolor{orange!10}pitchman,\\
cheer-     & \cellcolor{orange!10}nurse, &  & \cellcolor{orange!10}beauty queen, & getter, & \cellcolor{orange!10}villager, &  & \cellcolor{orange!10}interpreter, &  & jailer, & \cellcolor{orange!10}retired,\\
leader, & \cellcolor{orange!10}homemaker, &  & \cellcolor{orange!10}stripper & snowboarder & \cellcolor{orange!10}cricketer &  & \cellcolor{orange!10}smuggler &  & schoolboy & \cellcolor{orange!10}pundit\\
dietitian & \cellcolor{orange!10}chairwoman &  &  &  &  &  &  &  &  & \\
\hline
 & \cellcolor{orange!10}log cabin, & \cellcolor{orange!10}front porch, & racecourse, & \cellcolor{orange!10}picnic tables, & \cellcolor{orange!10}locality, & \cellcolor{orange!10}prefecture, &  & \cellcolor{orange!10}synagogues, & \cellcolor{orange!10}apartment & \\
 & \cellcolor{orange!10}library, & \cellcolor{orange!10}carport, & plush, & \cellcolor{orange!10}bleachers, & \cellcolor{orange!10}mosque, & \cellcolor{orange!10}chalet, &  & \cellcolor{orange!10}constructions, & \cellcolor{orange!10}complex, & \\
 & \cellcolor{orange!10}fairgrounds & \cellcolor{orange!10}duplex & tenements & \cellcolor{orange!10}concession & \cellcolor{orange!10}slum & \cellcolor{orange!10}sauna &  & \cellcolor{orange!10}hilltop & \cellcolor{orange!10}barbershop, & \\
 &  &  &  & \cellcolor{orange!10}stand &  &  &  &  & \cellcolor{orange!10}nightclub & \\
\hline
 & \cellcolor{orange!10}parish, & \cellcolor{orange!10}pastor, & \cellcolor{orange!10}goddess, &  & \cellcolor{orange!10}fatwa, & \cellcolor{orange!10}monastery, & \cellcolor{orange!10}rosary, & \cellcolor{orange!10}rabbis, &  & \\
 & \cellcolor{orange!10}church, & \cellcolor{orange!10}baptized, & \cellcolor{orange!10}celestial, &  & \cellcolor{orange!10}mosques, & \cellcolor{orange!10}papal, & \cellcolor{orange!10}parish priest, & \cellcolor{orange!10}synagogue, &  & \\
 & \cellcolor{orange!10}pastoral & \cellcolor{orange!10}mourners & \cellcolor{orange!10}mystical &  & \cellcolor{orange!10}martyrs & \cellcolor{orange!10}convent & \cellcolor{orange!10}patron saint & \cellcolor{orange!10}biblical &  & \\
\hline
\cellcolor{orange!10}volleyball, & athletic & leading & hooker, & \cellcolor{orange!10}sophomore, & leftarm &  &  &  & \cellcolor{orange!10}cornerback, & \\
\cellcolor{orange!10}gymnast, & director, & rebounder, & footy, & \cellcolor{orange!10}junior, & spinner, &  &  &  & \cellcolor{orange!10}tailback, & \\
\cellcolor{orange!10}setter & winningest & played & stud & \cellcolor{orange!10}freshman & dayers, &  &  &  & \cellcolor{orange!10}wide receiver & \\
 & coach, & sparingly, &  &  & leg spinner &  &  &  &  & \\
 & officiating & incoming &  &  &  &  &  &  &  & \\
 &  & freshman &  &  &  &  &  &  &  & \\
\hline
sorority, & \cellcolor{orange!10}volunteer, & guidance &  & \cellcolor{orange!10}seventh & lecturers, &  & \cellcolor{orange!10}bilingual, &  & incoming & \cellcolor{orange!10}fulltime,\\
gymnastics, & \cellcolor{orange!10}volunteering, & counselor, &  & \cellcolor{orange!10}grader, & institutes, &  & \cellcolor{orange!10}permanent &  & freshmen, & \cellcolor{orange!10}professional,\\
majoring & \cellcolor{orange!10}secretarial & prekinder- &  & \cellcolor{orange!10}eighth grade, & syllabus &  & \cellcolor{orange!10}residency, &  & schoolyard, & \cellcolor{orange!10}apprentice-\\
 &  & garten, &  & \cellcolor{orange!10}seniors &  &  & \cellcolor{orange!10}occupations &  & recruiting & \cellcolor{orange!10}ship\\
 &  & graduate &  &  &  &  &  &  &  & \\
\hline
 &  & \cellcolor{orange!10}civil rights, &  &  & \cellcolor{orange!10}subcontinent, & \cellcolor{orange!10}xenophobia, & \cellcolor{orange!10}leftist, & \cellcolor{orange!10}disengage- & \cellcolor{orange!10}blacks, & \\
 &  & \cellcolor{orange!10}poverty &  &  & \cellcolor{orange!10}tribesmen, & \cellcolor{orange!10}anarchist, & \cellcolor{orange!10}drug & \cellcolor{orange!10}ment, & \cellcolor{orange!10}segregation, & \\
 &  & \cellcolor{orange!10}stricken, &  &  & \cellcolor{orange!10}miscreants & \cellcolor{orange!10}oligarchs & \cellcolor{orange!10}traffickers, & \cellcolor{orange!10}intifada, & \cellcolor{orange!10}lynching & \\
 &  & \cellcolor{orange!10}nonviolent &  &  &  &  & \cellcolor{orange!10}undocumented & \cellcolor{orange!10}settlers &  & \\
\hline
tiara, & \cellcolor{orange!10}knitting, & \cellcolor{orange!10}brown eyes, & \cellcolor{orange!10}girly, & brown hair, & \cellcolor{orange!10}sari, &  &  &  & \cellcolor{orange!10}dreadlocks, & \cellcolor{orange!10}mullet,\\
blonde, & \cellcolor{orange!10}sewing, & \cellcolor{orange!10}cream colo..., & \cellcolor{orange!10}feminine, & pair, & \cellcolor{orange!10}turban, &  &  &  & \cellcolor{orange!10}shoulderpads, & \cellcolor{orange!10}gear,\\
sparkly & \cellcolor{orange!10}beaded & \cellcolor{orange!10}wore & \cellcolor{orange!10}flirty & skates & \cellcolor{orange!10}hijab &  &  &  & \cellcolor{orange!10}waistband & \cellcolor{orange!10}helmet\\
\hline
 &  &  &  &  & \cellcolor{orange!10}dirhams, & \cellcolor{orange!10}rubles, & \cellcolor{orange!10}pesos, & \cellcolor{orange!10}shekels, &  & \\
 &  &  &  &  & \cellcolor{orange!10}lakhs, & \cellcolor{orange!10}kronor, & \cellcolor{orange!10}remittances, & \cellcolor{orange!10}settlements, &  & \\
 &  &  &  &  & \cellcolor{orange!10}rupees & \cellcolor{orange!10}roulette & \cellcolor{orange!10}gross re- & \cellcolor{orange!10}corpus &  & \\
 &  &  &  &  &  &  & \cellcolor{orange!10}ceipts &  &  & \\
\hline
 &  & grandjury &  & child & \cellcolor{orange!10}chargesheet, & \cellcolor{orange!10}absentia, & \cellcolor{orange!10}illegal &  & \cellcolor{orange!10}aggravated & \\
 &  & indicted, &  & endangerment, & \cellcolor{orange!10}absconding, & \cellcolor{orange!10}tax evasion, & \cellcolor{orange!10}immigrant, &  & \cellcolor{orange!10}robbery, & \\
 &  & degree     &  & vehicular & \cellcolor{orange!10}interrogation & \cellcolor{orange!10}falsification & \cellcolor{orange!10}drug &  & \cellcolor{orange!10}aggravated & \\
 &  & murder, &  & homicide, &  &  & \cellcolor{orange!10}trafficking, &  & \cellcolor{orange!10}assault, & \\
 &  & violating &  & unlawful &  &  & \cellcolor{orange!10}deported &  & \cellcolor{orange!10}felonious & \\
 &  & probation &  & possession &  &  &  &  & \cellcolor{orange!10}assault & \\
\hline
 & \cellcolor{orange!10}volunteers, & \cellcolor{orange!10}caseworkers, & \cellcolor{orange!10}beauties, & setters, & \cellcolor{orange!10}mediapersons, &  &  &  & \cellcolor{orange!10}recruits, & \\
 & \cellcolor{orange!10}crafters, & \cellcolor{orange!10}evacuees, & \cellcolor{orange!10}celebs, & helpers, & \cellcolor{orange!10}office &  &  &  & \cellcolor{orange!10}reps, & \\
 & \cellcolor{orange!10}baby     & \cellcolor{orange!10}attendants & \cellcolor{orange!10}paparazzi & captains & \cellcolor{orange!10}bearers, &  &  &  & \cellcolor{orange!10}sheriffs & \\
 & \cellcolor{orange!10}boomers &  &  &  & \cellcolor{orange!10}newsmen &  &  &  &  & \\
\hline
\end{tabular*}

}
\cprotect\caption{The top-12 WEATs output by our UBE algorithm on the \verb|w2v| embedding. Columns represent name groups $X_i$ from Table \ref{table:fullw2v_names}, rows represent categories $A_j$ (e.g., a cluster of food-related words). Orange indicate associations where the crowd's most commonly chosen name group agrees with that of the generated WEAT. No significant biases generated for \textbf{w2v F10}. 
}\label{table:fullw2v}
\end{table*}

\subsection{Crowdsourcing Evaluation}

We solicited ratings on the biases generated by the algorithm from US-based crowd workers on Amazon's Mechanical Turk\footnote{\url{http://mturk.com}} platform. 
The aim is to identify whether the biases found by our UBE algorithm are consistent with (problematic) biases held by society at large. To this end, we asked about society's stereotypes, {\em not} personal beliefs.

We evaluated the top 12 WEATs generated by our UBE algorithm for the three embeddings, considering $n=12$ first name groups. Our approach was simple: after familiarizing participants with the 12 groups, we showed the (statistically significant) words and name groups of a WEAT and asked them to identify which words would, stereotypically, be most associated with which names group. A bonus was given for ratings that agreed with most other worker's ratings, incentivizing workers to provide answers that they felt corresponded to widely held stereotypes. 

This design was chosen over a simpler one in which WEATs are shown to individuals who are asked whether or not these are stereotypical. The latter design might support confirmation bias as people may interpret words in such a way that confirms whatever stereotypes they are being asked about. For instance, someone may be able to justify associating the color red with almost any group, a posteriori.

Note that the task presented to the workers involved fine-grained distinctions: for each of the top-12 WEATs, at least 18 workers would each be asked to match the significant $c \leq 12$ word triples to the $c$ name groups (each identified by five names each). For example, workers faced the triple of ``registered nurse, homemaker, chairwoman'' with $c=8$ groups of names, half of which were majority female, and the most commonly chosen group matched the one generated: ``Janice, Jeanette, Lenna, Mattie, Marylynn.''  Across the top-12 WEATs over the three embeddings, the mean number of choices $c$ was 8.1, yet the most commonly chosen group (plurality) agreed with the generated group 65\% of the time (see Table \ref{table:summary}). This is significantly more than one would expect from chance. The top-12 WEATs generated for \verb|w2v| are shown in Table \ref{table:fullw2v}.


One challenge faced in this process was that, in pilot experiments, a significant fraction of the workers were not familiar with many of the names. To address this challenge, we first administered a qualification exam (common in crowdsourcing) in which each worker was shown 36 random names, 3 from each group, and was offered a bonus for each name they could correctly identify the group from which it was chosen. Only workers whose accuracy was greater than 1/2 (which happened 37\% of the time) evaluated the WEATs. Accuracy greater than 50\% on a 12-way classification indicates that the groups of names were meaningful and interpretable to many workers.

Finally, we asked 13-15 workers to rate associations on a scale of 1-7 of {\em political incorrectness}, with 7 being ``politically incorrect, possibly very offensive'' and 1 being ``politically correct, inoffensive, or just random.'' Only those biases for which the most commonly chosen group matched the association identified by the UBE algorithm were included in this experiment. The mean ratings are shown in Table \ref{table:summary} and the terms present in associations deemed most offensive are presented in Table~\ref{table:offensive}.

\subsection{Potential Indirect Biases and Proxies}

Naively, one may think that removing names from a dataset will remove all problematic associations. However, as suggested by \cite{Bolukbasi16:Man}, indirect biases are likely to remain. For example, 
consider the \verb|w2v| word embedding, in which {\em hostess} is closer to {\em volleyball} than to {\em cornerback}, while {\em cab driver} is closer to {\em cornerback} than to {\em volleyball}. These associations, taken from columns \textbf{F1} and \textbf{F11} of Table \ref{table:fullw2v}, might serve as a proxy for gender and/or race. For instance, if someone is applying for a job and their profile includes college sports words, such associations encoded in the embedding may lead to racial or gender biases in cases in which there is no professional basis for these associations. In contrast, {\em volunteer} being closer to {\em volunteers} than {\em recruits} may represent a definitional similarity more than a proxy, if we consider proxies to be associations that mainly have predictive power due to their correlation with a protected attribute. While defining proxies is beyond the scope of this work, we do say that $A_{ij}, A_{i'j}, A_{ij'}, A_{i'j'}$ is a {\em potential indirect bias} if,
\begin{equation}\label{eq:potproxy}
    (\bar{A}_{ij}-\bar{A}_{i'j})\cdot(\bar{A}_{ij'}-\bar{A}_{i'j'})>0. 
\end{equation}
One way to interpret this definition is that if the embedding were to match the pair of word sets $\{A_{ij}, A_{i'j}\}$ to the pair of word sets $\{A_{ij'}, A_{i'j'}\}$, it would align with the way in which they were generated. For example, does the embedding predict that {\em hostess-cab driver} better fits {\em volleyball-cornerback} or {\em cornerback-volleyball} (but this question is asked with sets of $t=3$ words)?
Downstream, this would mean that a replacing a the word {\em cornerback} with {\em volleyball} on a profile would make it closer to {\em hostess} than {\em cab driver}
 
We consider all possible fourtuples of significant associations, such that $1\leq i<i'\leq n$ and $1\leq j<j'\leq m$. In the case of \verb|w2v|, $99\%$ of 2,713 significant fourtuples lead to potential indirect biases according to eq.~(\ref{eq:potproxy}). This statistic is of 98\% of 1,125 fourtuples and 97\% of 1,796 fourtuples for the \verb|fast| and \verb|glove| embeddings, respectively. Hence, while names allow us to capture biases in the embedding, removing names is unlikely to be sufficient to debias the embedding.


\section{Limitations}

Absent clusters show the limitations of our approach and data.
For example, even for large $n$, no clusters represent demographically significant Asian-American groups. However, if instead of names we use surnames  \cite[U.S.~Census,][]{comenetz2016frequently}, a cluster ``Yu, Tamashiro, Heng, Feng, Nakamura, +393'' emerges, which is largely Asian according to Census data (see Table \ref{table:w2vlast} in the Appendix). This distinction may reflect naming practices among Asian Americans \citep{wu1999jones}. 
Similarly, our approach may miss biases against small minorities or other groups whose names are not significantly differentiated. For example, it is not immediately clear to what extent this methodology can capture biases against individuals whose gender identity is non-binary, although interestingly terms associated with transgender individuals were generated and rated as significant and consistent with human biases.

\section{Conclusions and Discussion}\label{sec:disc}

We introduce the problem of Unsupervised Bias Enumeration. We propose and evaluate a UBE algorithm that outputs Word Embedding Association Tests. Unlike humans, where implicit tests are necessary to elicit socially unacceptable biases in a straightforward fashion, word embeddings can be directly probed to output hundreds of biases of varying natures, including numerous offensive and socially unacceptable biases. 

The racist and sexist associations exposed in publicly available word embeddings raise questions about their widespread use. An important open question is how to reduce these biases.

\smallskip
\noindent
\textbf{Acknowledgments}. We are grateful to Tarleton Gillespie, other colleagues and the anonymous reviewers for useful feedback. 



{\fontsize{10pt}{10.5pt}\selectfont
\bibliography{fairidentify}
\bibliographystyle{aaai}}

\newpage
\onecolumn
\appendix

\section{Offensive Stereotypes and Derogatory Terms}\label{ap:bleep}

The authors consulted with colleagues whether to display the offensive terms and stereotypes that emerged from the embedding using our algorithms. First, regarding derogatory terms, people we consulted found the explicit inclusion of some of these terms offensive. We are also sensitive to the fact that, even in investigating them, we are ourselves using them. The terms we bleep-censor in the tables include slurs regarding race, homosexuality, transgender, and mental ability \citep{bianchi2014slurs}. In particular, these include three variants on ``the n word''  \citep{asim2008n}, {\em shemale}, {\em faggot}, {\em twink}, {\em mentally retarded}, and {\em rednecks}. It is not obvious that such slurs would be generated given common naming conventions. Nonetheless, many of these terms were in groups of words that matched stereotypes indicated by crowd workers. 

Of course, the associations of words and groups are also offensive, but unfortunately, it is impossible to convey the nature of these associations without presenting the words in the tables associated with the groups. In an attempt to soften the effect, we use group letters rather than illustrative names or summary statistics in our tables. While this decreases the transparency, it gives the reader a choice about whether or not to examine the associated names.
Some colleagues were taken aback by an initial draft, in which names and associations were displayed in the same table, and it was noted that it that may be especially offensive to individuals
whose name appeared on top of a column of offensive stereotypes.
For the names, we restrict our selection of names to those that had at least 1,000 occurrences in the data so that the name would not be uniquely identified with any individual. 

In addition, we considered withholding the entire tables and merely presenting the rating statistics. However, we decided that, given that our concern in the analysis is uncovering that such troubling associations are being made by these tools, it was important to be 
clear and unflinching
about what we found, and not risk obscuring the very phenomenon in our explanation. 



\section{Proofs of Lemmas}\label{ap:proofs}
\begin{proof}[Proof of Lemma \ref{lem:decomp}]
For $n=2$, using our $\bar{X}$ notation and their assumption $|X_1|=|X_2|$, simple algebra shows that,
$$(\bar{X}_1-\bar{X}_2) \cdot (\bar{A}_1-\bar{A}_2)=\frac{1}{|X_1|}s(X_1,A_1,X_2,A_2).$$
Since $\vmu = (\bar{X}_1+\bar{X}_2)/2$, we have that $\bar{X}_1-\vmu=(\bar{X_1}-\bar{X}_2)/2=-(\bar{X}_2-\vmu)$, and:
\begin{align*}
g(X_1,A_1,X_2,A_2) &= (\bar{X}_1 - \vmu) \cdot (\bar{A}_1-\bar{\mathcal{A}}) +(\bar{X}_2 - \vmu) \cdot (\bar{A}_2-\bar{\mathcal{A}})\\
&= \frac{\bar{X}_1-\bar{X}_2}{2} \cdot \bigl(\bar{A}_1-\bar{\mathcal{A}}-(\bar{A}_2-\bar{\mathcal{A}})\bigr)\\
&= \frac{1}{2} (\bar{X}_1-\bar{X}_2) \cdot (\bar{A}_1-\bar{A}_2),
\end{align*}
which when combined with the previous equality establishes the first equation in Lemma \ref{lem:decomp}.
\end{proof}

\begin{proof}[Proof of Lemma \ref{lem:decomp2}] Since we have shown that $(\bar{X}_1-\bar{X}_2) \cdot (\bar{A}_1-\bar{A}_2) = 2g(X_1,A_1,X_2,A_2)$ above, we immediately have that $g(X, A) = 2 g(X,A,\mathcal{X},\mathcal{A})$. Moreover, simple algebra shows  that $g(X,A,\mathcal{X},\mathcal{A})$ and $g(X, A, X^c, A^c)$ are proportional because $\bar{X}-\bar{\mathcal{X}}=\frac{|X^c|}{|\mathcal{X}|}(\bar{X}-\bar{X^c})$ and similarly $\bar{A}-\bar{\mathcal{A}}=\frac{|A^c|}{|\mathcal{A}|}(\bar{A}-\bar{A^c})$.
\end{proof}
\begin{proof}[Proof of Lemma \ref{lem:decomp3}] Follows simply from the definition of $g$ and $\mu$ for $n\geq 2$ and $n=1$.
\end{proof}

\section{Preprocessing names and words}\label{ap:preprocessing}

\subsection{Preprocessing first names from SSA dataset}
The SSA dataset \citep{SSA} has partial coverage for earlier years and includes all names with at least 5 births, we use only years 1938-2017 and select only the names that appeared at least 1,000 times, which cover more than 99\% of the data by population. From this data, we extract the fraction of female and male births for each name as well as the mean year of birth. Of course, we select only the names appearing in the embedding.

Note that the mean of the fraction of females among our names is significantly greater than 50\%, even though the US population is nearly balanced in binary gender demographics. The subtle reason is there is greater variability in female names in the data, whereas the most common names are more often male. That is, the data have fewer predominantly male first names in total with more people being given those names on average. Since we are including each name only once, this increases the female representation in the population.\footnote{We performed similar experiments on a sample of names drawn according to the population and, while the names are gender balanced, the clusters exhibit less diversity and most often simply are split by gender and age -- one can even have an entire cluster solely consisting of people named {\em Michael}.}

\subsection{Preprocessing last names from U.S.~Census}

A dataset of last names is made publicly available by the Census Bureau of the United States and contains last names occurring at least 100 times in the 2010 census~\citep{comenetz2016frequently}, broken down by percentage of race, including White, Black, Hispanic, Asian and Pacific Islander, and Native American. Again we filter for names that appear at least 1,000 times and apply the binary classification procedure described in Section \ref{sec:clean} to clean the data. 
\subsection{``Cleaning'' names}

\cite{Caliskan17:Semantics} 
apply a simple procedure in which they remove the 20\% of words whose mean similarity to the other names is smallest. We apply a similar but slightly more sophisticated procedure by training an linear Support Vector Machine \citep[scikit-learn's LinearSVC,][with default parameters]{scikit-learn} to distinguish the input names from an equal number of non-names chosen randomly from the most frequent 50,000 words in the embedding. We then remove the 20\% of names with smallest margin in the direction identified by the linear classifier.

Figure \ref{fig:powerlaw} illustrates the effect of cleaning the last names and shows that the names that tend to be removed are those that violate Zipf's law. 

\begin{figure}
	\centering
\includegraphics[width=0.6\columnwidth]{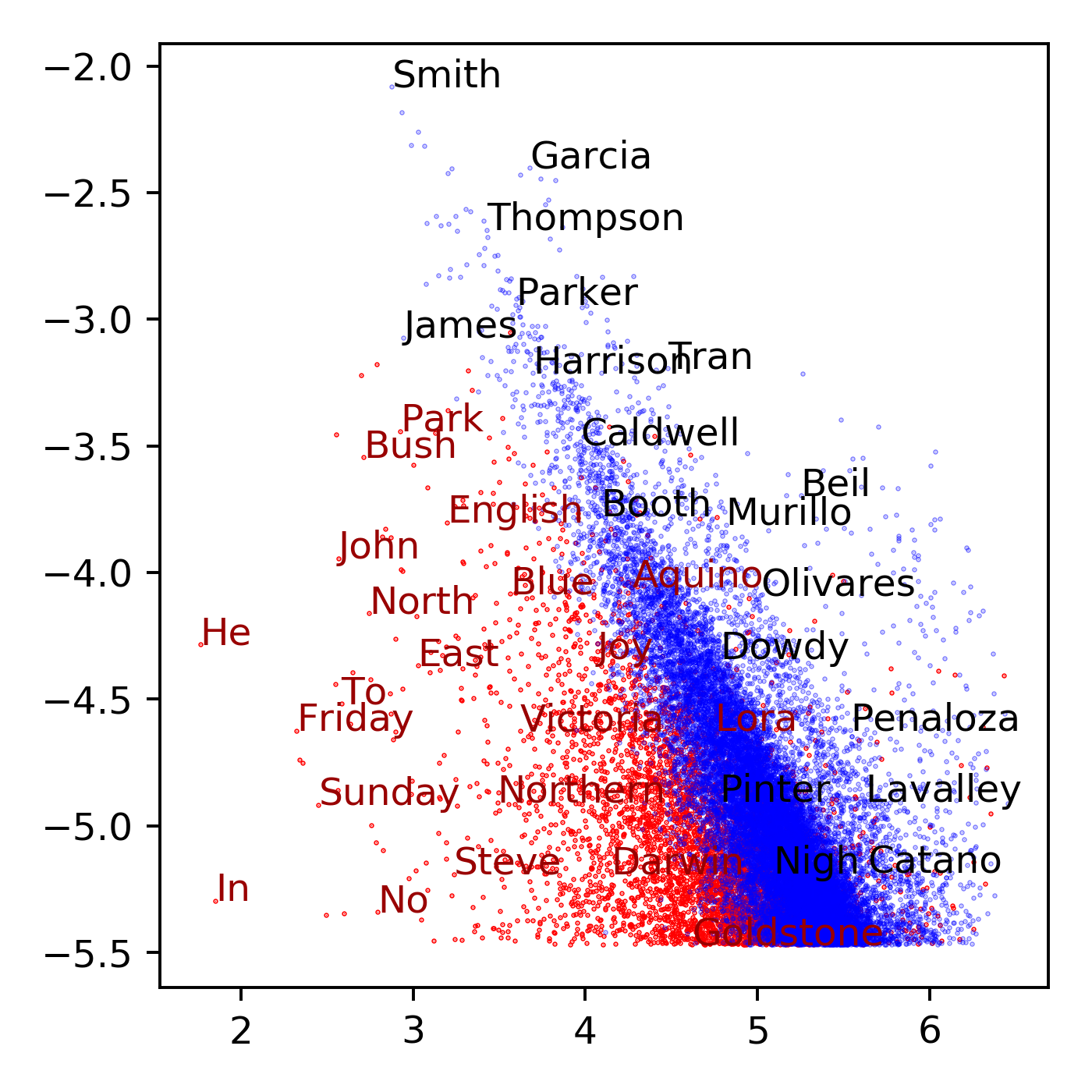}
\caption{A plot of log-probability (y-axis) vs.\ word embedding index (x-axis) for the last name data and the word2vec word embedding. Orange points represent last names we keep and blue points are outliers we remove. As expected from Zipf's law, the probabilities and frequencies exhibit a power-law relationship. Names removed from the data by our classifier, displayed in red, are typically words that have other more common uses than as last names.}
\label{fig:powerlaw}

\end{figure}

\subsection{Preprocessing words}
To identify the most frequent $M$ words in the embedding,
we first restrict to tokens that consist only of the 26 lower-case English letters or spaces for embeddings that contain phrases. We also omit lower-case tokens when the upper-case version of the token is more frequent. For instance, the lower-case token ``john'' is removed because ``John'' is more frequent.

\section{Biases in different lists/embeddings}\label{ap:otherbiases}

Table \ref{table:names_other_embeddings} shows the names from other embeddings. Table \ref{table:debiased} shows the biases found in the ``debiased'' \verb|w2v| embedding of \cite{Bolukbasi16:Man}, while Table \ref{table:w2vlast} show last-name biases generated from the \verb|w2v| embeddings.

\begin{table*}
\small{
\setlength\tabcolsep{1pt}
\begin{tabular*}{\textwidth}{@{\extracolsep{\fill}}rrrrrrrrrrrr}
\toprule
\textbf{fast~F1} & \textbf{fast~F2} & \textbf{fast~F3} & \textbf{fast~F4} & \textbf{fast~F5} & \textbf{fast~F6} & \textbf{fast~F7} & \textbf{fast~F8} & \textbf{fast~F9} & \textbf{fast~F10} & \textbf{fast~F11} & \textbf{fast~F12} \\
\midrule
Nakesha & Carolyn & Tamara & Lillian & Alejandra & Katelyn & Ahmed & Landon & Stephan & Marquell & Greg & Gerardo \\
Keisha & Nichole & Emi & Lucinda & Maricella & Jayda & Shanti & Keenan & Nahum & Antwan & Willie & Renato \\
Kandyce & Mel & Isabella & Velda & Ona & Shalyn & Mariyah & Skye & Sabastian & Dakari & Edward & Pedro \\
Kamilah & Tawnya & Karina & Antoinette & Fabiola & Jaylyn & Siddharth & Courtland & Philippe & Pernell & Jefferey & Genaro \\
Rachal & Deirdre & Joli & Flossie & Sulema & Evie & Yasmin & Luke & Jarek & Jarred & Russ & Matteo \\
~~~+702 & ~~~+821 & ~~~+622 & ~~~+478 & ~~~+400 & ~~~+851 & ~~~+288 & ~~~+576 & ~~~+312 & ~~~+440 & ~~~+474 & ~~~+234 \\
\midrule
98\% F & 98\% F & 97\% F & 96\% F & 93\% F & 90\% F & 64\% F & 22\% F &  9\% F &  6\% F &  4\% F &  2\% F \\
\hline
1980 & 1972 & 1987 & 1972 & 1984 & 1993 & 1992 & 1991 & 1987 & 1984 & 1973 & 1987 \\
\hline
29\% B &  4\% B &  5\% B & 14\% B &  2\% B &  3\% B &  6\% B &  5\% B &  6\% B & 34\% B &  8\% B &  1\% B \\
 3\% H &  2\% H &  9\% H &  9\% H & 64\% H &  2\% H &  4\% H &  1\% H &  9\% H &  3\% H &  3\% H & 65\% H \\
 1\% A &  2\% A &  6\% A &  6\% A &  8\% A &  2\% A & 33\% A &  3\% A &  4\% A &  2\% A &  5\% A &  7\% A \\
66\% W & 91\% W & 80\% W & 71\% W & 25\% W & 93\% W & 56\% W & 90\% W & 80\% W & 61\% W & 84\% W & 27\% W \\
\midrule
\end{tabular*}

\medskip

\setlength\tabcolsep{1pt}
\begin{tabular*}{\textwidth}{@{\extracolsep{\fill}}rrrrrrrrrrrr}
\toprule
\textbf{glove~F1} & \textbf{glove~F2} & \textbf{glove~F3} & \textbf{glove~F4} & \textbf{glove~F5} & \textbf{glove~F6} & \textbf{glove~F7} & \textbf{glove~F8} & \textbf{glove~F9} & \textbf{glove~F10} & \textbf{glove~F11} & \textbf{glove~F12} \\
\midrule
Elsie & Brenda & Claudia & Patrica & Kylee & Laticia & Alejandra & Amina & Eldridge & Damion & Kevin & Gustavo \\
Carlotta & Katie & Tiara & Caren & Shaye & Jayci & Epifanio & Yair & Tad & Ronney & Ernest & Etienne \\
Elizabeth & Janette & Lena & Mikala & Tayla & Shalanda & Monalisa & Rani & Godfrey & Winford & Haley & Lorenzo \\
Dovie & Liza & Melina & Cherise & Latasha & Kalynn & Eulalia & Danial & Asa & Tavaris & Matt & Emil \\
Gladys & Debra & Sasha & Lorine & Jessi & Noelani & Alicea & Safa & Renard & Tylor & Gilbert & Roberto \\
~~~+263 & ~~~+396 & ~~~+359 & ~~~+889 & ~~~+520 & ~~~+1270 & ~~~+395 & ~~~+396 & ~~~+434 & ~~~+627 & ~~~+429 & ~~~+218 \\
\midrule
99\% F & 98\% F & 95\% F & 94\% F & 89\% F & 83\% F & 68\% F & 58\% F & 18\% F & 11\% F &  7\% F &  6\% F \\
\hline
1972 & 1974 & 1987 & 1973 & 1987 & 1978 & 1985 & 1989 & 1979 & 1982 & 1979 & 1987 \\
\hline
15\% B &  4\% B &  6\% B &  7\% B &  9\% B & 14\% B &  1\% B &  5\% B & 13\% B & 11\% B &  7\% B &  3\% B \\
11\% H &  3\% H & 12\% H &  3\% H &  3\% H & 28\% H & 67\% H &  4\% H &  3\% H &  2\% H &  3\% H & 41\% H \\
 6\% A &  3\% A &  7\% A &  2\% A &  3\% A &  2\% A &  9\% A & 22\% A &  4\% A &  2\% A &  4\% A &  6\% A \\
68\% W & 89\% W & 73\% W & 88\% W & 85\% W & 55\% W & 22\% W & 68\% W & 80\% W & 84\% W & 85\% W & 50\% W \\
\midrule
\end{tabular*}

\medskip

\setlength\tabcolsep{1pt}
\begin{tabular*}{\textwidth}{@{\extracolsep{\fill}}rrrrrrrrrrrr}
\toprule
\textbf{deb.~F1} & \textbf{deb.~F2} & \textbf{deb.~F3} & \textbf{deb.~F4} & \textbf{deb.~F5} & \textbf{deb.~F6} & \textbf{deb.~F7} & \textbf{deb.~F8} & \textbf{deb.~F9} & \textbf{deb.~F10} & \textbf{deb.~F11} & \textbf{deb.~F12} \\
\midrule
Denise & Kayla & Evelyn & Marquisha & Zoe & Kamal & Nicolas & Luis & Michal & Shaneka & Randall & Brian \\
Audrey & Lynae & Marquetta & Madalynn & Nana & Nailah & Carmella & Deisy & Astrid & Dondre & Scarlett & Ernie \\
Maryalice & Gabe & Gaylen & Celene & Crystal & Kalan & Adrien & Alexandro & Ezra & Laquanda & Windell & Matthew \\
Sonja & Tayla & Gaye & Nyasia & Georgiana & Aisha & Stefania & Elsa & Armen & Tavon & Corrin & Kenny \\
Glenna & Staci & Eula & Lanora & Sariyah & Rony & Raphael & Eliazar & Juliane & Tanesha & Coley & Wayne \\
~~~+714 & ~~~+845 & ~~~+506 & ~~~+819 & ~~~+512 & ~~~+334 & ~~~+322 & ~~~+538 & ~~~+282 & ~~~+688 & ~~~+407 & ~~~+313 \\
\midrule
99\% F & 81\% F & 80\% F & 78\% F & 71\% F & 62\% F & 59\% F & 56\% F & 54\% F & 49\% F & 29\% F &  5\% F \\
\hline
1971 & 1989 & 1969 & 1984 & 1984 & 1991 & 1984 & 1986 & 1987 & 1983 & 1982 & 1974 \\
\hline
 4\% B &  4\% B & 17\% B &  5\% B & 10\% B &  6\% B &  6\% B &  1\% B &  2\% B & 49\% B &  9\% B &  5\% B \\
 3\% H &  3\% H &  6\% H &  3\% H &  9\% H &  5\% H & 16\% H & 72\% H &  6\% H &  3\% H &  3\% H &  3\% H \\
 3\% A &  2\% A &  4\% A &  3\% A & 11\% A & 32\% A &  5\% A &  8\% A &  3\% A &  2\% A &  4\% A &  5\% A \\
89\% W & 91\% W & 72\% W & 89\% W & 70\% W & 56\% W & 73\% W & 18\% W & 88\% W & 45\% W & 83\% W & 87\% W \\
\midrule
\end{tabular*}

\medskip

\setlength\tabcolsep{1pt}
\begin{tabular*}{\textwidth}{@{\extracolsep{\fill}}rrrrrrrrrrrr}
\toprule
\textbf{w2v~L1} & \textbf{w2v~L2} & \textbf{w2v~L3} & \textbf{w2v~L4} & \textbf{w2v~L5} & \textbf{w2v~L6} & \textbf{w2v~L7} & \textbf{w2v~L8} & \textbf{w2v~L9} & \textbf{w2v~L10} & \textbf{w2v~L11} & \textbf{w2v~L12} \\
\midrule
Moser & Stein & Boyer & Romano & Murphy & Cantrell & Gauthier & Burgess & Gaines & Lal & Mendez & Yu \\
Persson & Zucker & Lasher & Klimas & Nagle & Wooddell & Medeiros & Willson & Derouen & Haddad & Aguillon & Tamashiro \\
Pagel & Avakian & Sawin & Pecoraro & Igoe & Maness & Lafrance & Hatton & Gaskins & Mensah & Aispuro & Heng \\
Runkel & Sobel & Stoudt & Arnone & Crosbie & Newcomb & Lounsbury & Mutch & Aubrey & Vora & Forero & Feng \\
Wagner & Tepper & Mcintire & Morreale & Dillon & Greathouse & Renard & Patten & Rodgers & Omer & Jurado & Nakamura \\
~~~+3035 & ~~~+775 & ~~~+3013 & ~~~+1416 & ~~~+665 & ~~~+2444 & ~~~+756 & ~~~+2818 & ~~~+1779 & ~~~+423 & ~~~+1913 & ~~~+393 \\
\midrule
 1\% B &  2\% B &  3\% B &  1\% B &  4\% B &  8\% B &  8\% B & 12\% B & 34\% B & 15\% B &  1\% B &  1\% B \\
 2\% H &  3\% H &  2\% H &  6\% H &  3\% H &  2\% H &  4\% H &  3\% H &  3\% H &  7\% H & 80\% H &  3\% H \\
 1\% A &  1\% A &  1\% A &  1\% A &  1\% A &  1\% A &  1\% A &  1\% A &  1\% A & 28\% A &  5\% A & 79\% A \\
94\% W & 93\% W & 92\% W & 91\% W & 90\% W & 86\% W & 85\% W & 81\% W & 60\% W & 46\% W & 12\% W & 11\% W \\
\midrule
\end{tabular*}

}
\cprotect\caption{The first name clusters from the \verb|fast|, \verb|glove| and \verb|debiased| embeddings, followed by last name clusters from the \verb|w2v| embedding.
Demographic statistics (computed a
posteriori) are also shown though were not used in generation, including percentage female (at birth), mean year of birth, and
percentage Black, Hispanic, Asian/Pacific Islander, and White.\label{table:names_other_embeddings}
}
\end{table*}

\begin{table*}
{\fontsize{8.8}{10.2}\selectfont
\setlength\tabcolsep{1pt}
\begin{tabular*}{\textwidth}{@{\extracolsep{\fill}}llllllllll}
\toprule
\textbf{deb.~F1} & \textbf{deb.~F2} & \textbf{deb.~F3} & \textbf{deb.~F4} & \textbf{deb.~F5} & \textbf{deb.~F6} & \textbf{deb.~F7} & \textbf{deb.~F8} & \textbf{deb.~F9} & \textbf{deb.~F10} \\
\midrule
professor & eighth grader, & lifelong & granddaughter, & bloke, & shopkeeper, & mobster, & translator, & mathematician, & cousin,\\
emeritus, & seventh & resident, & grandson, & chap, & villager, & chef, & interpreter, & physicist, & jailer,\\
registered & grader, & postmaster, & daughter & hubby & elder brother & restaurateur & notary & researcher & roommate\\
nurse, & sixth grader & homemaker &  &  &  &  &  &  & \\
adjunct &  &  &  &  &  &  &  &  & \\
professor &  &  &  &  &  &  &  &  & \\
\hline
volunteering, & seniors, & grandparents, & graduated, & bedtime, & expatriate, &  & undocumented, &  & blacks,\\
homebound, & eighth grade, & aunts, & grandchildren, & marital, & hostels, &  & farmworkers, &  & academically,\\
nurse & boys & elderly & siblings & bisexual & postgraduate &  & bilingual &  & mentally\\
practitioner &  &  &  &  &  &  &  &  & r********\\
\hline
 & medley, & bluegrass, & trombone, &  & artiste, & maestro, & flamenco, & avant garde, & rapper,\\
 & solo, & bandleader, & percussionist, &  & verse, & accordion, & tango, & violinist, & gospel,\\
 & trio & banjo & clarinet &  & remix & operas & vibes & techno & hip hop\\
\hline
 & volleyball, & bass fishing, & wearing & racecourse, & cricket, & peloton, &  & luge, & basketball,\\
 & softball, & rodeo, & helmet, & footy, & badminton, & anti doping, &  & biathlon, & sprints,\\
 & roping & deer hunting & horseback & footballing & cricketing & gondola &  & chess & lifting\\
 &  &  & riding, &  &  &  &  &  & weights\\
 &  &  & snorkeling &  &  &  &  &  & \\
\hline
 &  & rural, & westbound, & foreshore, & slum, & seaside, & barangays, & settlements, & \\
 &  & fairgrounds, & southbound, & tenements, & headquarter, & boutiques, & squatters, & prefecture, & \\
 &  & tract & eastbound & tourist & minarets & countryside & plazas & inhabitants & \\
 &  &  &  & attraction &  &  &  &  & \\
\hline
 &  & supper, & macaroni, &  & halal, & pizzeria, & tortillas, & kosher, & \\
 &  & barbecue, & green beans, &  & sweets, & mozzarella, & salsa, & vodka, & \\
 &  & chili & pancakes &  & hummus & pasta & tequila & bagel & \\
\hline
 &  &  &  &  & dirhams, & euros, & peso, & supervisory & \\
 &  &  &  &  & emirate, & francs, & reais, & board, & \\
 &  &  &  &  & riyals & vintages & nationalized & zloty, & \\
 &  &  &  &  &  &  &  & ruble & \\
\hline
 &  & pastor, & baptized, & mystical, & fatwa, & nuns, &  & rabbis, & \\
 &  & church, & sisters, & witch, & mosque, & papal, &  & synagogue, & \\
 &  & parish & brothers & afterlife & martyrs & monastery &  & commune & \\
\hline
 & captains, & caretakers, & cousins, & punters, & mediapersons, &  &  &  & rappers,\\
 & bridesmaids, & grandmothers, & helpers, & blokes, & office &  &  &  & recruits,\\
 & grads & superinten- & friends & celebs & bearers, &  &  &  & officers\\
 &  & dents &  &  & shopkeepers &  &  &  & \\
\hline
 &  &  & clan, &  & subcontinent, &  & leftist, & rightist, & civil rights,\\
 &  &  & overthrow, &  & rulers, &  & indigenous & disengage- & segregation,\\
 &  &  & starvation &  & tribals &  & peoples, & ment, & racial\\
 &  &  &  &  &  &  & peasants & oligarchs & \\
\hline
 &  &  &  &  & rupees, &  & pesos, & shekels, & \\
 &  &  &  &  & dinars, &  & remittances, & rubles, & \\
 &  &  &  &  & crores &  & cooperatives & kronor & \\
\hline
 &  & convicted & child &  & chargesheet, & absentia, &  &  & aggravated\\
 &  & felon, & endangerment, &  & absconding, & annulment, &  &  & robbery,\\
 &  & felony & unlawful &  & petitioner & penitentiary &  &  & aggravated\\
 &  & convictions, & possession, &  &  &  &  &  & assault,\\
 &  & probate & vehicular &  &  &  &  &  & felonious\\
 &  &  & homicide &  &  &  &  &  & assault\\
\hline
\end{tabular*}

}
\cprotect\caption{The top-12 WEATs output by our UBE algorithm on the ``debiased'' \verb|w2v| embedding of \cite{Bolukbasi16:Man}, again with $n=12$. Despite being debiased, demographic statistics (again computed a posteriori) reveal names still cluster by gender, but the extreme gender clusters have many fewer statistically significant associations. For instance, the most male groups \textbf{deb.~F11} and \textbf{deb.~F12} are not shown because no significant associations were generated.
}\label{table:debiased}
\end{table*}



\begin{table*}
{\fontsize{8.6}{10}\selectfont
\setlength\tabcolsep{1pt}
\begin{tabular*}{\textwidth}{@{\extracolsep{\fill}}llllllllllll}
\toprule
\textbf{w2v~L1} & \textbf{w2v~L2} & \textbf{w2v~L3} & \textbf{w2v~L4} & \textbf{w2v~L5} & \textbf{w2v~L6} & \textbf{w2v~L7} & \textbf{w2v~L8} & \textbf{w2v~L9} & \textbf{w2v~L10} & \textbf{w2v~L11} & \textbf{w2v~L12} \\
\midrule
potato & kosher, & pumpkin, & mozzarella, & pint, & pecans, & maple   & cider, & fried      & sweets, & tortillas, & noodles,\\
salad, & bagel, & brownies, & pasta, & whiskey, & grits, & syrup, & lager, & chicken, & saffron, & salsa, & dumplings,\\
pretzels, & hummus & donuts & deli & cheddar & watermelon & syrup, & malt & crawfish, & mango & tequila & soy sauce\\
chocolate &  &  &  &  &  & foie gras &  & sweet &  &  & \\
cake &  &  &  &  &  &  &  & potatoes &  &  & \\
\hline
concentra- & disengage- &  &  & unionists, &  & province, & antisocial & blacks, & non & drug & hyun,\\
tion & ment, &  &  & sectarian, &  & separatist, & behavior, & segrega- & governmen- & traffick- & bian,\\
camp, & neocons, &  &  & pedophiles &  & sover- & cricket, & tion, & tal, & ers, & motherland\\
extermina- & intifada &  &  &  &  & eignty & asylum & civil & miscreants, & leftist, & \\
tion, &  &  &  &  &  &  & seekers & rights & encroach- & undocu- & \\
postwar &  &  &  &  &  &  &  &  & ments & mented & \\
\hline
 & co founder, & assessor, & restaura- & solicitor, & jailer, &  & schoolboy, & cheer-     & shopkeeper, & translator, & villager,\\
 & venture & wildlife & teur, & selector, & rancher, &  & barrister, & leader, & aspirant, & smuggler, & vice,\\
 & capitalist, & biologist, & plumber, & handicap- & appraiser &  & chap & bailiff, & taxi & inter- & housewife\\
 & psycho- & secretary & fire- & per &  &  &  & recruiter & driver & preter & \\
 & therapist & treasurer & fighter &  &  &  &  &  &  &  & \\
\hline
 & synagogues, & log cabin, & pizzeria, & pubs, & fair- & rink, & disused, &  & locality, &  & prefecture,\\
 & skyscraper, & zoning & borough, & racecourse, & grounds, & cottage, & derelict, &  & slum, &  & guesthouse,\\
 & studio & ordinance, & firehouse & western & acre tract, & chalet & leisure &  & hostel &  & metropolis\\
 &  & barn &  & suburbs & concession &  &  &  &  &  & \\
 &  &  &  &  & stand &  &  &  &  &  & \\
\hline
 & authors, & crafters, & mobsters, & gardai, & sheriffs, & skaters, & blokes, &  & mediaper- &  & migrant\\
 & hedgefund & hobbyists, & restaura- & lads, & folks, & premiers, & household- &  & sons, &  & workers,\\
 & managers, & racers & teurs, & foot- & appraisers & mushers & ers, &  & newsmen, &  & maids,\\
 & creators &  & captains & ballers &  &  & solicitors &  & office &  & civil\\
 &  &  &  &  &  &  &  &  & bearers &  & servants\\
\hline
 & rabbis, &  & papal, & archdio- & denomina- &  & vicar, & pulpit, & fatwa, & rosary, & commune,\\
 & synagogue, &  & pontiff, & cese, & tion, &  & creation- & preaching, & fasting, & parish & monks,\\
 & biblical &  & convent & clerical, & pastor, &  & ism, & preach & sufferings & priest, & temples\\
 &  &  &  & diocese & church &  & tradition- &  &  & patron & \\
 &  &  &  &  &  &  & alists &  &  & saint & \\
\hline
 & shekels, & mill levy, &  &  & millage, &  & unfair &  & rupees, & pesos, & baht,\\
 & settle- & assessed &  &  & payday &  & dismissal, &  & lakhs, & remit- & overseas,\\
 & ments, & valuation, &  &  & lenders, &  & atten- &  & dirhams & tances, & income\\
 & nonprofit & tax &  &  & appropria- &  & dances, &  &  & indigent & earners\\
 &  & abatement &  &  & tions &  & takings &  &  &  & \\
\hline
 & pollster, &  &  &  & commis- & ridings, &  & desegrega- & panchayat, & barangay, & plenary\\
 & liberal, &  &  &  & sioners, & selectmen, &  & tion, & candida- & immigra- & session,\\
 & moderates &  &  &  & countywide, & byelection &  & uncommit- & ture, & tion & landslide,\\
 &  &  &  &  & statewide &  &  & ted, & localities & reform, & multira-\\
 &  &  &  &  &  &  &  & voter &  & congress- & cial\\
 &  &  &  &  &  &  &  & registra- &  & woman & \\
 &  &  &  &  &  &  &  & tion &  &  & \\
\hline
 & insider & felonious &  &  & sheriff, & impaired & affray, & aggravated & absconding, & illegal & \\
 & trading, & assault, &  &  & meth lab, & driving, & bailiffs, & robbery, & charge- & immigrant, & \\
 & attorneys, & drug &  &  & jailers & criminal & aggravated & racially & sheet, & drug & \\
 & lawsuit & parapher- &  &  &  & negligence, & burglary & charged, & com- & traffick- & \\
 &  & nalia, &  &  &  & peniten- &  & probation & plainant & ing, & \\
 &  & criminal &  &  &  & tiary &  & violation &  & deadly & \\
 &  & mischief &  &  &  &  &  &  &  & weapon & \\
\hline
 &  &  &  &  &  & loonie, & sharemar- &  & load & peso, & cross\\
 &  &  &  &  &  & francs, & ket, &  & shedding, & reais, & strait,\\
 &  &  &  &  &  & takeovers & credit &  & microfi- & national- & yuan,\\
 &  &  &  &  &  &  & crunch, &  & nance, & ization & ringgit\\
 &  &  &  &  &  &  & gilts &  & rupee &  & \\
\hline
walleye, & transat- &  &  &  & crappie, &  &  & shad, & mangroves, & sardines, & mainland,\\
lakes, & lantic, &  &  &  & bass &  &  & barrier & jetty, & tuna, & seaweed,\\
aquarium & iceberg, &  &  &  & fishing, &  &  & islands, & kite & archipel- & island\\
 & flotilla &  &  &  & boat ramp &  &  & grouper &  & ago & \\
\hline
feedlot, &  & cornfield, &  &  & mowing, &  &  &  & agro, & farmwork- & bamboo,\\
barley, &  & pumpkins, &  &  & deer &  &  &  & saplings, & ers, & cassava,\\
wheat &  & alfalfa &  &  & hunting, &  &  &  & livelihood & coca, & palm oil\\
 &  &  &  &  & pasture &  &  &  &  & sugarcane & \\
\hline
\end{tabular*}

}
\cprotect\caption{The top-12 WEATs output by our UBE algorithm on the \verb|w2v| embedding for {\em last names}. The corresponding name groups are presented in Table \ref{table:names_other_embeddings}.
}\label{table:w2vlast}
\end{table*}








\end{document}